\documentclass[english]{elsarticle}
\usepackage[T1]{fontenc}
\usepackage[latin9]{inputenc}
\usepackage{amsmath}
\usepackage{amssymb}
\usepackage{graphicx}

\makeatletter
\usepackage{amsmath, amsthm}

\renewcommand\[{\begin{equation}}
\renewcommand\]{\end{equation}} 

\newtheorem{theorem}{Theorem}[section]

\newtheorem{lemma}[theorem]{Lemma}

\newtheorem{definition}[theorem]{Definition}

\newcommand{\interior}[1]{%
  {\kern0pt#1}^{\mathrm{o}}%
}

\makeatother

\usepackage{babel}
\begin{document}

\title{Piecewise convexity of artificial neural networks}

\author{Blaine Rister$_{a}$, Daniel L. Rubin$_{b}$}

\address{$_{a}$Stanford University, Department of Electrical Engineering.
email: blaine@stanford.edu. Corresponding author.\\
$_{b}$Stanford University, Department of Radiology (Biomedical Informatics
Research). email: dlrubin@stanford.edu.\\
1201 Welch Rd\\
Stanford, CA, USA 94305}
\begin{abstract}
Although artificial neural networks have shown great promise in applications
including computer vision and speech recognition, there remains considerable
practical and theoretical difficulty in optimizing their parameters.
The seemingly unreasonable success of gradient descent methods in
minimizing these non-convex functions remains poorly understood. In
this work we offer some theoretical guarantees for networks with piecewise
affine activation functions, which have in recent years become the
norm. We prove three main results. Firstly, that the network is piecewise
convex as a function of the input data. Secondly, that the network,
considered as a function of the parameters in a single layer, all
others held constant, is again piecewise convex. Finally, that the
network as a function of all its parameters is piecewise multi-convex,
a generalization of biconvexity. From here we characterize the local
minima and stationary points of the training objective, showing that
they minimize certain subsets of the parameter space. We then analyze
the performance of two optimization algorithms on multi-convex problems:
gradient descent, and a method which repeatedly solves a number of
convex sub-problems. We prove necessary convergence conditions for
the first algorithm and both necessary and sufficient conditions for
the second, after introducing regularization to the objective. Finally,
we remark on the remaining difficulty of the global optimization problem.
Under the squared error objective, we show that by varying the training
data, a single rectifier neuron admits local minima arbitrarily far
apart, both in objective value and parameter space.
\end{abstract}
\maketitle{}


\section{Introduction}

Artificial neural networks are currently considered the state of the
art in applications ranging from image classification, to speech recognition
and even machine translation. However, little is understood about
the process by which they are trained for supervised learning tasks.
The problem of optimizing their parameters is an active area both
practical and theoretical research. Despite considerable sensitivity
to initialization and choice of hyperparameters, neural networks often
achieve compelling results after optimization by gradient descent
methods. Due to the nonconvexity and massive parameter space of these
functions, it is poorly understood how these sub-optimal methods have
proven so successful. Indeed, training a certain kind of neural network
is known to be NP-Complete, making it difficult to provide any worst-case
training guarantees \cite{Blum:1992:OCT:148433.148441}. Much recent
work has attempted to reconcile these differences between theory and
practice \cite{Kawaguchi:2016:WithoutPoorLocalMinima,Soudry:2016:NoBadLocalMinima}.

This article attempts a modest step towards understanding the dynamics
of the training procedure. We establish three main convexity results
for a certain class of neural network, which is the current the state
of the art. First, that the objective is piecewise convex as a function
of the input data, with parameters fixed, which corresponds to the
behavior at test time. Second, that the objective is again piecewise
convex as a function of the parameters of a single layer, with the
input data and all other parameters held constant. Third, that the
training objective function, for which all parameters are variable
but the input data is fixed, is piecewise multi-convex. That is, it
is a continuous function which can be represented by a finite number
of multi-convex functions, each active on a multi-convex parameter
set. This generalizes the notion of biconvexity found in the optimization
literature to piecewise functions and arbitrary index sets \cite{Gorski:2007:BiConvex}.
To prove these results, we require two main restrictions on the definition
of a neural network: that its layers are piecewise affine functions,
and that its objective function is convex and continuously differentiable.
Our definition includes many contemporary use cases, such as least
squares or logistic regression on a convolutional neural network with
rectified linear unit (ReLU) activation functions and either max-
or mean-pooling. In recent years these networks have mostly supplanted
the classic sigmoid type, except in the case of recurrent networks
\cite{Glorot:2011:ReluNetworks}. We make no assumptions about the
training data, so our results apply to the current state of the art
in many practical scenarios.

Piecewise multi-convexity allows us to characterize the extrema of
the training objective. As in the case of biconvex functions, stationary
points and local minima are guaranteed optimality on larger sets than
we would have for general smooth functions. Specifically, these points
are partial minima when restricted to the relevant piece. That is,
they are points for which no decrease can be made in the training
objective without simultaneously varying the parameters across multiple
layers, or crossing the boundary into a different piece of the function.
Unlike global minima, we show that partial minima are reliably found
by the optimization algorithms used in current practice.

Finally, we provide some guarantees for solving general multi-convex
optimization problems by various algorithms. First we analyze gradient
descent, proving necessary convergence conditions. We show that every
point to which gradient descent converges is a piecewise partial minimum,
excepting some boundary conditions. To prove stronger results, we
define a different optimization procedure breaking each parameter
update into a number of convex sub-problems. For this procedure, we
show both necessary and sufficient conditions for convergence to a
piecewise partial minimum. Interestingly, adding regularization to
the training objective is all that is needed to prove necessary conditions.
Similar results have been independently established for many kinds
of optimization problems, including bilinear and biconvex optimization,
and in machine learning the special case of linear autoencoders \cite{Wendell:1976:Bilinear,Gorski:2007:BiConvex,Baldi:2012:ComplexValuedAutoencoders}.
Our analysis extends existing results on alternating convex optimization
to the case of arbitrary index sets, and general multi-convex point
sets, which is needed for neural networks. We admit biconvex problems,
and therefore linear autoencoders, as a special case.

Despite these results, we find that it is difficult to pass from partial
to global optimality results. Unlike the encouraging case of linear
autoencoders, we show that a single rectifier neuron, under the squared
error objective, admits arbitrarily poor local minima. This suggests
that much work remains to be done in understanding how sub-optimal
methods can succeed with neural networks. Still, piecewise multi-convex
functions are in some senses easier to minimize than the general class
of smooth functions, for which none of our previous guarantees can
be made. We hope that our characterization of neural networks could
contribute to a better understanding of these important machine learning
systems.

\section{Preliminary material}

We begin with some preliminary definitions and basic results concerning
continuous piecewise functions.

\begin{definition}\label{def:piecewise_affine}

Let $g_{1},g_{2},...,g_{N}$ be continuous functions from $\mathbb{R}^{n}\rightarrow\mathbb{R}$.
A \textbf{continuous piecewise} function $f$ has a finite number
of closed, connected sets $S_{1},S_{2},...,S_{M}$ covering $\mathbb{R}^{n}$
such that for each $k$ we have $f(\boldsymbol{x})=g_{k}(\boldsymbol{x})$
for all $\boldsymbol{x}\in S_{k}$. The set $S_{k}$ is called a \textbf{piece
}of $f$, and the function $g_{k}$ is called \textbf{active} on $S_{k}$.

More specific definitions follow by restricting the functions $g$.
A \textbf{continuous piecewise affine }function has $g_{k}(\boldsymbol{x})=\boldsymbol{a}^{T}\boldsymbol{x}+b$
where $\boldsymbol{a}\in\mathbb{R}^{n}$ and $b\in\mathbb{R}$. A
\textbf{continuous piecewise convex} function has $g_{k}$ convex,
with $S_{k}$ convex as well.

\end{definition}

Note that this definition of piecewise convexity differs from that
found in the convex optimization literature, which focuses on \textit{convex}
piecewise convex functions, i.e.~maxima of convex functions \cite{Tsevendorj:2001:PiecewiseConvex}.
Note also that we do not claim a unique representation in terms of
active functions $g_{k}$ and pieces $S_{k}$, only that there exists
at least one such representation.

Before proceeding, we shall extend definition \ref{def:piecewise_affine}
to functions of multidimensional codomain for the affine case.

\begin{definition}\label{def:piecewise_affine_onto_Rn}

A function $f:\mathbb{R}^{m}\rightarrow\mathbb{R}^{n}$, and let $f_{k}:\mathbb{R}^{m}\rightarrow\mathbb{R}$
denote the $k^{th}$ component of $f$. Then $f$ is\textbf{ continuous
piecewise affine} if each $f_{k}$ is. Choose some piece $S_{k}$
from each $f_{k}$ and let $S=\cap_{k=1}^{n}S_{k}$, with $S\ne\emptyset$.
Then $S$ is a piece of $f$, on which we have $f(\boldsymbol{x})=A\boldsymbol{x}+\boldsymbol{b}$
for some $A\in\mathbb{R}^{n\times m}$ and $\boldsymbol{b}\in\mathbb{R}^{n}$.

\end{definition}

First, we prove an intuitive statement about the geometry of the pieces
of continuous piecewise affine functions.

\begin{theorem}\label{theorem:convex_polytope}

Let $f:\mathbb{R}^{m}\rightarrow\mathbb{R}^{n}$ be continuous piecewise
affine. Then $f$ admits a representation in which every piece is
a convex polytope.

\end{theorem}

\begin{proof}

Let $f_{k}:\mathbb{R}^{m}\rightarrow\mathbb{R}$ denote the $k^{th}$
component of $f$. Now, $f_{k}$ can be written in closed form as
a max-min polynomial \cite{Ovchinnikov:2002:pwl_max_min}. That is,
$f_{k}$ is the maximum of minima of its active functions. Now, for
the minimum of two affine functions we have
\[
\min(g_{i},g_{j})=\min(\boldsymbol{a}_{i}^{T}\boldsymbol{x}+b_{i},\boldsymbol{a}_{j}^{T}\boldsymbol{x}+b_{j}).
\]
This function has two pieces divided by the hyperplane $(\boldsymbol{a}_{i}^{T}-\boldsymbol{a}_{j}^{T})\boldsymbol{x}+b_{i}-b_{j}=0$.
The same can be said of $\max(g_{i},g_{j})$. Thus the pieces of $f_{k}$
are intersections of half-spaces, which are just convex polytopes.
Since the pieces of $f$ are intersections of the pieces of $f_{k}$,
they are convex polytopes as well.

\end{proof}

See figure \ref{fig:bad_local_minimum} in section \ref{sec:Local-minima}
for an example of this result on a specific neural network. Our next
result concerns the composition of piecewise functions, which is essential
for the later sections.

\begin{theorem}\label{theorem:composition_piecewise_affine}

Let $g:\mathbb{R}^{m}\rightarrow\mathbb{R}^{n}$ and $f:\mathbb{R}^{n}\rightarrow\mathbb{R}$
be continuous piecewise affine. Then so is $f\circ g$.

\end{theorem}

\begin{proof}

To establish continuity, note that the composition of continuous functions
is continuous.

Let $S$ be a piece of $g$ and $T$ a piece of $f$ such that $S\cap g^{-1}(T)\ne\emptyset$,
where $g^{-1}(T)$ denotes the inverse image of $T$. By theorem \ref{theorem:convex_polytope},
we can choose $S$ and $T$ to be convex polytopes. Since $g$ is
affine, $g^{-1}(T)$ is closed and convex \cite{Boyd:2004:ConvexOptimization}.
Thus $S\cap g^{-1}(T)$ is a closed, convex set on which we can write
\begin{align}
f(\boldsymbol{x}) & =\boldsymbol{a}^{T}\boldsymbol{x}+b\label{eq:pwa1}\\
g(\boldsymbol{x}) & =C\boldsymbol{x}+\boldsymbol{d}.\nonumber 
\end{align}
Thus
\begin{align}
f\circ g(\boldsymbol{x}) & =\boldsymbol{a}^{T}C\boldsymbol{x}+\boldsymbol{a}^{T}\boldsymbol{d}+b\label{eq:pwa2}
\end{align}
which is an affine function.

Now, consider the finite set of all such pieces $S\cap g^{-1}(T)$.
The union of $g^{-1}(T)$ over all pieces $T$ is just $\mathbb{R}^{n}$,
as is the union of all pieces $S$. Thus we have
\begin{align*}
\cup_{S}\cup_{T}\left(S\cap g^{-1}(T)\right) & =\cup_{S}\left(S\cap\cup_{T}g^{-1}(T)\right)\\
 & =\cup_{S}\left(S\cap\mathbb{R}^{n}\right)\\
 & =\mathbb{R}^{n}.
\end{align*}
 Thus $f\circ g$ is piecewise affine on $\mathbb{R}^{n}$.

\end{proof}

We now turn to continuous piecewise convex functions, of which continuous
piecewise affine functions are a subset.

\begin{theorem}\label{theorem:covex_piecewise}

Let $g:\mathbb{R}^{m}\rightarrow\mathbb{R}^{n}$ be a continuous piecewise
affine function, and $h:\mathbb{R}^{n}\rightarrow\mathbb{R}$ a convex
function. Then $f=h\circ g$ is continuous piecewise convex.

\end{theorem}

\begin{proof}

On each piece $S$ of $g$ we can write
\[
f(\boldsymbol{x})=h(A\boldsymbol{x}+\boldsymbol{b}).
\]

This function is convex, as it is the composition of a convex and
an affine function \cite{Boyd:2004:ConvexOptimization}. Furthermore,
$S$ is convex by theorem \ref{theorem:convex_polytope}. This establishes
piecewise convexity by the proof of theorem \ref{theorem:composition_piecewise_affine}.

\end{proof}

Our final theorem concerns the arithmetic mean of continuous piecewise
convex functions, which is essential for the analysis of neural networks.

\begin{theorem}\label{theorem:mean_piecewise_convex}

Let $f_{1},f_{2},...,f_{N}$ be continuous piecewise convex functions.
Then so is their arithmetic mean $(1/N)\sum_{i=1}^{N}f_{i}(\boldsymbol{x})$.

\end{theorem}

The proof takes the form of two lemmas.

\begin{lemma}\label{lemma:sum_piecewise_convex}

Let $f_{1}$ and $f_{2}$ be a pair of continuous piecewise convex
functions on $\mathbb{R}^{n}$. Then so is $f_{1}+f_{2}$.

\end{lemma}

\begin{proof}

Let $S_{1}$ be a piece of $f_{1}$, and $S_{2}$ a piece of $f_{2}$,
with $S_{1}\cap S_{2}\ne\emptyset$. Note that the sum of convex functions
is convex \cite{Rockafellar:1970:ConvexAnalysis}. Thus $f_{1}+f_{2}$
is convex on $S_{1}\cap S_{2}$. Furthermore, $S_{1}\cap S_{2}$ is
convex because it is an intersection of convex sets \cite{Rockafellar:1970:ConvexAnalysis}.
Since this holds for all pieces of $f_{1}$ and $f_{2}$, we have
that $f_{1}+f_{2}$ is continuous piecewise convex on $\mathbb{R}^{n}$.

\end{proof}

\begin{lemma}

Let $\alpha>0$, and let $f$ be a continuous piecewise convex function.
Then so is $\alpha f$.

\end{lemma}

\begin{proof}

The continuous function $\alpha f$ is convex on every piece of $f$.

\end{proof}

Having established that continuous piecewise convexity is closed under
addition and positive scalar multiplication, we can see that it is
closed under the arithmetic mean, which is just the composition of
these two operations.

\section{Neural networks\label{sec:Neural-networks}}

In this work, we define a neural network to be a composition of functions
of two kinds: a convex continuously differentiable objective (or loss)
function $h$, and continuous piecewise affine functions $g_{1},g_{2},...,g_{N}$,
constituting the $N$ layers. Furthermore, the outermost function
must be $h$, so that we have
\[
f=h\circ g_{N}\circ g_{N-1}\circ...\circ g_{1}
\]
where $f$ denotes the entire network. This definition is not as restrictive
as it may seem upon first glance. For example, it is easily verified
that the rectified linear unit (ReLU) neuron is continuous piecewise
affine, as we have 
\[
g(\boldsymbol{x})=\max(\boldsymbol{0},A\boldsymbol{x}+\boldsymbol{b}),
\]
where the maximum is taken pointwise. It can be shown that maxima
and minima of affine functions are piecewise affine \cite{Ovchinnikov:2002:pwl_max_min}.
This includes the convolutional variant, in which $A$ is a Toeplitz
matrix. Similarly, max pooling is continuous piecewise linear, while
mean pooling is simply linear. Furthermore, many of the objective
functions commonly seen in machine learning are convex and continuously
differentiable, as in least squares and logistic regression. Thus
this seemingly restrictive class of neural networks actually encompasses
the current state of the art.

By theorem \ref{theorem:composition_piecewise_affine}, the composition
of all layers $g=g_{N}\circ g_{N-1}\circ...\circ g_{1}$ is continuous
piecewise affine. Therefore, a neural network is ultimately the composition
of a continuous convex function with a single continuous piecewise
affine function. Thus by theorem \ref{theorem:covex_piecewise} the
network is continuous piecewise convex. Figure \ref{fig:network}
provides a visualization of this result for the example network
\begin{equation}
f(x,y)=\left(2-\left[\left[x-y\right]_{+}-\left[x+y\right]_{+}+1\right]_{+}\right)^{2},\label{eq:example_network}
\end{equation}
where $\left[x\right]_{+}=\max(x,0)$. For clarity, this is just the
two-layer ReLU network
\[
f(x,y,z)=\left(z-\left[a_{5}\left[a_{1}x+a_{2}y\right]_{+}+a_{6}\left[a_{3}x+a_{4}y\right]_{+}+b_{1}\right]_{+}\right)^{2}
\]
with the squared error objective and a single data point $((x,y),z)$,
setting $z=2$ and $a_{2}=a_{6}=-1$, with all other parameters set
to $1$.

\begin{figure}
\begin{centering}
\includegraphics[scale=0.5]{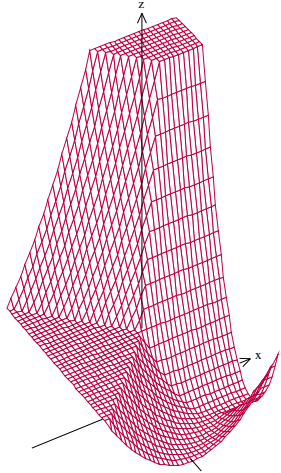}
\par\end{centering}
\centering{}\caption{The neural network of equation \ref{eq:example_network}, on the unit
square. Although $f$ is not convex on $\mathbb{R}^{2}$, it is convex
in each piece, and each piece is a convex set.}
\label{fig:network}
\end{figure}
Before proceeding further, we must define a special kind of differentiability
for piecewise continuous functions, and show that this holds for neural
networks.

\begin{definition}

Let $f$ be piecewise continuous. We say that $f$ is \textbf{piecewise
continuously differentiable} if each active function $g$ is continuously
differentiable.

\end{definition}

To see that neural networks are piecewise continuously differentiable,
note that the objective $h$ is continuously differentiable, as are
the affine active functions of the layers. Thus their composition
is continuously differentiable. It follows that non-differentiable
points are found only on the boundaries between pieces.

\section{Network parameters of a single layer\label{sec:Network-parameters-of-single-layer}}

In the previous section we have defined neural networks as functions
of labeled data. These are the functions relevant during testing,
where parameters are constant and data is variable. In this section,
we extend these results to the case where data is constant and parameters
are variable, which is the function to optimized during training.
For example, consider the familiar equation
\[
f=(ax+b-y)^{2}
\]
with parameters $(a,b)$ and data $(x,y$). During testing, we hold
$(a,b)$ constant, and consider $f$ as a function of the data $(x,y)$.
During training, we hold $(x,y)$ constant and consider $f$ as a
function of the parameters $(a,b)$. This is what we mean when we
say that a network is being ``considered as a function of its parameters\footnote{This is made rigorous by taking cross-sections of point sets in section
\ref{sec:Network-parameters-of-multiple-layers}. }.'' This leads us to an additional stipulation on our definition
of a neural network. That is, each layer must be piecewise affine
\textit{as a function of its parameters} as well. This is easily verified
for all of the layer types previously mentioned. For example, with
the ReLU neuron we have
\begin{equation}
f(A,\boldsymbol{b})=\left[A\boldsymbol{x}+\boldsymbol{b}\right]_{+}\label{eq:ReLU}
\end{equation}
 so for $\left(A\boldsymbol{x}+\boldsymbol{b}\right)_{k}\ge0$ we
have that the $k^{th}$ component of $f$ is linear in $(A,\boldsymbol{b})$,
while for $\left(A\boldsymbol{x}+\boldsymbol{b}\right)_{k}<0$ it
is constant. To see this, we can re-arrange the elements of $A$ into
a column vector $\boldsymbol{a}$, in row-major order, so that we
have
\begin{align}
A\boldsymbol{x}+\boldsymbol{b} & =\begin{pmatrix}\boldsymbol{x^{T}} & \boldsymbol{0^{T}} & ... & ... & \boldsymbol{0^{T}} & \boldsymbol{1}^{T}\\
\boldsymbol{0^{T}} & \boldsymbol{x^{T}} & \boldsymbol{0}^{T} & ... & \boldsymbol{0}^{T} & \boldsymbol{1}^{T}\\
... & ... & ... & ... & ... & ...\\
\boldsymbol{0}^{T} & ... & ... & \boldsymbol{0}^{T} & \boldsymbol{x}^{T} & \boldsymbol{1}^{T}
\end{pmatrix}\begin{pmatrix}\boldsymbol{a}\\
\boldsymbol{b}
\end{pmatrix}.\label{eq:ReLU_expanded}
\end{align}

In section \ref{sec:Neural-networks} we have said that a neural network,
considered as a function of its input data, is convex and continuously
differentiable on each piece. Now, a neural network need \textit{not}\textbf{
}be piecewise convex as a function of the entirety of its parameters\footnote{To see this, consider the following two-layer network: $h(x)=x$,
$g_{2}(x)=ax$, and $g_{1}(x)=bx$. For $f=h\circ g_{2}\circ g_{1}$
we have $f(x)=abx$. Now fix the input as $x=1$. Considered as a
function of its parameters, this is $f(a,b)=ab$, which is decidedly
not convex.}. However, we can regain piecewise convexity by considering it only
as a function of the parameters in a single layer, all others held
constant.

\begin{theorem}\label{theorem:neural_network_single_layer}

A neural network $f$ is continuous piecewise convex and piecewise
continuously differentiable as a function of the parameters in a single
layer.

\end{theorem}

\begin{proof}

For the time being, assume the input data consists of a single point
$\boldsymbol{x}$. By definition $f$ is the composition of a convex
objective $h$ and layers $g_{1},g_{2},...,g_{N}$, with $g_{1}$
a function of $\boldsymbol{x}$. Let $f_{m}(\boldsymbol{x})$ denote
the network $f$ considered as a function of the parameters of layer
$g_{m}$, all others held constant. Now, the layers $g_{m-1}\circ g_{m-2}\circ...\circ g_{1}$
are constant with respect to the parameters of $g_{m}$, so we can
write $\boldsymbol{y}=g_{m-1}\circ g_{m-2}\circ...\circ g_{1}(\boldsymbol{x})$.
Thus on each piece of $g_{m}$ we have
\[
g_{m}=A\boldsymbol{y}+\boldsymbol{b}.
\]
By definition $g_{m}$ is a continuous piecewise affine function of
its parameters. Since $\boldsymbol{y}$ is constant, we have that
$\tilde{g}_{m}=g_{m}\circ g_{m-1}\circ...\circ g_{1}$ is a continuous
piecewise affine function of the parameters of $g_{m}$. Now, by theorem
\ref{theorem:composition_piecewise_affine} we have that $g=g_{N}\circ g_{N-1}\circ...\circ\tilde{g}_{m}$
is a continuous piecewise affine function of the parameters of $g_{m}$.
Thus by theorem \ref{theorem:covex_piecewise}, $f_{m}$ is continuous
piecewise convex.

To establish piecewise continuous differentiability, recall that affine
functions are continuously differentiable, as is $h$.

Having established the theorem for the case of a single data point,
consider the case where we have multiple data points, denoted $\{\boldsymbol{x}_{k}\}_{k=1}^{M}$.
Now, by theorem \ref{theorem:mean_piecewise_convex} the arithmetic
mean $(1/M)\sum_{k=1}^{M}f_{m}(\boldsymbol{x}_{k})$ is continuous
piecewise convex. Furthermore, the arithmetic mean preserves piecewise
continuous differentiability. Thus these results hold for the mean
value of the network over the dataset.

\end{proof}

We conclude this section with a simple remark which will be useful
in later sections. Let $f_{m}$ be a neural network, considered as
a function of the parameters of the $m^{th}$ layer, and let $S$
be a piece of $f_{m}$. Then the optimization problem
\begin{align}
\mbox{minimize } & f_{m}(\boldsymbol{x})\nonumber \\
\mbox{subject to } & \boldsymbol{x}\in S\label{eq:convex_single_layer}
\end{align}
is convex.

\section{Network parameters of multiple layers\label{sec:Network-parameters-of-multiple-layers}}

In the previous section we analyzed the convexity properties of neural
networks when optimizing the parameters of a single layer, all others
held constant. Now we are ready to extend these results to the ultimate
goal of simultaneously optimizing all network parameters. Although
not convex, the problem has a special convex substructure that we
can exploit in proving future results. We begin by defining this substructure
for point sets and functions.

\begin{definition}

Let $S\subseteq\mathbb{R}^{n}$, let $I\subset\{1,2,...,n\}$, and
let $\boldsymbol{x}\in S$. The set
\[
S_{I}(\boldsymbol{x})=\{\boldsymbol{y}\in S:(\boldsymbol{y}_{k}=\boldsymbol{x}_{k})_{k\notin I}\}
\]
is the \textbf{cross-section} of $S$ intersecting $\boldsymbol{x}$
with respect to $I$.

\end{definition}

In other words, $S_{I}(\boldsymbol{x})$ is the subset of $S$ for
which every point is equal to $\boldsymbol{x}$ in the components
not indexed by $I$. Note that this differs from the typical definition,
which is the intersection of a set with a hyperplane. For example,
$\mathbb{R}_{\{1\}}^{3}(\boldsymbol{0})$ is the $x$-axis, whereas
$\mathbb{R}_{\{1,2\}}^{3}(\boldsymbol{0})$ is the $xy$-plane. Note
also that cross-sections are not unique, for example $\mathbb{R}_{\{1,2\}}^{3}(0,0,0)=\mathbb{R}_{\{1,2\}}^{3}(1,2,0)$.
In this case the first two components of the cross section are irrelevant,
but we will maintain them for notational convenience. We can now apply
this concept to functions on $\mathbb{R}^{n}$.

\begin{definition}

Let $S\subseteq\mathbb{R}^{n}$, let $f:S\rightarrow\mathbb{R}$ and
let $\mathcal{I}$ be a collection of sets covering $\{1,2,...,n\}$.
We say that $f$ is \textbf{multi-convex} with respect to $\mathcal{I}$
if $f$ is convex when restricted to the cross section $S_{I}(\boldsymbol{x})$,
for all $\boldsymbol{x}\in S$ and $I\in\mathcal{I}$.

\end{definition}

This formalizes the notion of restricting a non-convex function to
a variable subset on which it is convex, as in section \ref{sec:Network-parameters-of-single-layer}
when a neural network was restricted to the parameters of a single
layer. For example, let $f(x,y,z)=xy+z$, and let $I_{1}=\{1,3\}$,
and $I_{2}=\{2,3\}$. Then $f_{1}(x,y_{0},z)$ is a convex function
of $(x,z)$ with $y$ fixed at $y_{0}$. Similarly, $f_{2}(x_{0},y,z)$
is a convex function of $(y,z)$ with $x$ fixed at $x_{0}$. Thus
$f$ is multi-convex with respect to $\mathcal{I}=\{I_{1},I_{2}\}$.
To fully define a multi-convex optimization problem, we introduce
a similar concept for point sets.

\begin{definition}

Let $S\subseteq\mathbb{R}^{n}$ and let $\mathcal{I}$ be a collection
of sets covering $\{1,2,...,n\}$. We say that $S$ is \textbf{multi-convex
}with respect to $\mathcal{I}$ if the cross-section $S_{I}(\boldsymbol{x})$
is convex for all $\boldsymbol{x}\in S$ and $I\in\mathcal{I}$.

\end{definition}

This generalizes the notion of biconvexity found in the optimization
literature \cite{Gorski:2007:BiConvex}. From here, we can extend
definition \ref{def:piecewise_affine} to multi-convex functions.
However, we will drop the topological restrictions on the pieces of
our function, since multi-convex sets need not be connected.

\begin{definition}

Let $f:\mathbb{R}^{n}\rightarrow\mathbb{R}$ be a continuous function.
We say that $f$ is \textbf{continuous piecewise multi-convex} if
each there exists a collection of multi-convex functions $g_{1},g_{2},...,g_{N}$
and multi-convex sets $S_{1},S_{2},...,S_{N}$ covering $\mathbb{R}^{n}$
such that for each $k$ we have $f(\boldsymbol{x})=g_{k}(\boldsymbol{x})$
for all $\boldsymbol{x}\in S_{k}$. Next, let $h:\mathbb{R}^{m}\rightarrow\mathbb{R}^{n}$.
Then, $h$ is continuous piecewise multi-convex so long as each component
is, as in definition \ref{def:piecewise_affine_onto_Rn}.

\end{definition}

From this definition, it is easily verified that a continuous piecewise
multi-convex function $f:\mathbb{R}^{m}\rightarrow\mathbb{R}^{n}$
admits a representation where all pieces are multi-convex, as in the
proof of theorem \ref{theorem:convex_polytope}.

Before we can extend the results of section \ref{sec:Network-parameters-of-single-layer}
to multiple layers, we must add one final constraint on the definition
of a neural network. That is, each of the layers must be continuous
piecewise multi-convex, considered as functions of both the parameters
\textit{and} the input. Again, this is easily verified for the all
of the layer types previously mentioned. We have already shown they
are piecewise convex on each cross-section, taking our index sets
to separate the parameters from the input data. It only remains to
show that the number of pieces is finite. The only layer which merits
consideration is the ReLU, which we can see from equation \ref{eq:ReLU}
consists of two pieces for each component: the ``dead'' or constant
region, with $(A\boldsymbol{x})_{j}+b_{j}<0$, and its compliment.
With $n$ components we have at most $2^{n}$ pieces, corresponding
to binary assignments of ``dead'' or ``alive'' for each component.

Having said that each layer is continuous piecewise multi-convex,
we can extend these results to the whole network.

\begin{theorem}

Let $f$ be a neural network, and let $\mathcal{I}$ be a collection
of index sets, one for the parameters of each layer of $f$. Then
$f$ is continuous piecewise multi-convex with respect to $\mathcal{I}$.

\end{theorem}

We begin the proof with a lemma for more general multi-convex functions.

\begin{lemma}\label{lemma:composition_multi_convex}

Let $X\subseteq\mathbb{R}^{n}$, $Y\subseteq\mathbb{R}^{m}$, and
let $g:X\rightarrow Z$ and $f:Z\times Y\rightarrow\mathbb{R}^{n}$
be continuous piecewise multi-convex, $g$ with respect to a collection
of index sets $\mathcal{G}$, and $f$ with respect to $\mathcal{F}=\{I_{Z},I_{Y}\}$,
where $I_{Z}$ indexes the variables in $Z$, and $I_{Y}$ the variables
in $Y$. Then $h(\boldsymbol{x},\boldsymbol{y})=f(g(\boldsymbol{x}),\boldsymbol{y})$
is continuous piecewise multi-convex with respect to $\mathcal{H}=\mathcal{G}\cup\{I_{Y}\}$.

\end{lemma}

\begin{proof}

Let $G$ be a piece of $g$, let $F$ be a piece of $f$ and let $H=\{(\boldsymbol{x},\boldsymbol{y}):\boldsymbol{x}\in G,\,(g(\boldsymbol{x}),\boldsymbol{y})\in F\}$,
with $F$ chosen so that $H\ne\emptyset$. Clearly $h$ is multi-convex
on $H$ with respect to $\mathcal{H}$. It remains to show that $H$
is a multi-convex set. Now, let $(\boldsymbol{x},\boldsymbol{y})\in H$
and we shall show that the cross-sections are convex. First, for any
$I_{X}\in\mathcal{G}$ we have $H_{I_{X}}(\boldsymbol{x},\boldsymbol{y})=G_{I_{X}}(\boldsymbol{x})\times\{\boldsymbol{y}\}$.
Similarly, we have $H_{I_{Y}}(\boldsymbol{x},\boldsymbol{y})=\{\boldsymbol{x}\}\times\{\boldsymbol{y}:(\boldsymbol{z},\boldsymbol{y})\in F_{I_{Y}}(g(\boldsymbol{x}),\boldsymbol{y})\}$.
These sets are convex, as they are the Cartesian products of convex
sets \cite{Rockafellar:1970:ConvexAnalysis}. Finally, as in the proof
of theorem \ref{theorem:composition_piecewise_affine}, we can cover
$X\times Y$ with the finite collection of all such pieces $H$, taken
over all $G$ and $F$. 

\end{proof}

Our next lemma extends theorem \ref{theorem:mean_piecewise_convex}
to multi-convex functions.

\begin{lemma}\label{lemma:sum_piecewise_multi_convex}

Let $\mathcal{I}$ be a collection of sets covering $\{1,2,...,n\}$,
and let $f:\mathbb{R}^{n}\rightarrow\mathbb{R}$ and $g:\mathbb{R}^{n}\rightarrow\mathbb{R}$
be continuous piecewise multi-convex with respect to $\mathcal{I}$.
Then so is $f+g$.

\end{lemma}

\begin{proof}

Let $F$ be a piece of $f$ and $G$ be a piece of $g$ with $\boldsymbol{x}\in F\cap G$.
Then for all $I\in\mathcal{I}$, $\left(F\cap G\right)_{I}(\boldsymbol{x})=F_{I}(\boldsymbol{x})\cap G_{I}(\boldsymbol{x})$,
a convex set on which $f+g$ is convex. Thus $f+g$ is continuous
piecewise multi-convex, where the pieces of $f+g$ are the intersections
of pieces of $f$ and $g$.

\end{proof}

We can now prove the theorem.

\begin{proof}

For the moment, assume we have only a single data point. Now, let
$g_{1}$ and $g_{2}$ denote layers of $f$, with parameters $\boldsymbol{\theta}_{1}\in\mathbb{R}^{m},\,\boldsymbol{\theta}_{2}\in\mathbb{R}^{n}$.
Since $g_{1}$ and $g_{2}$ are continuous piecewise multi-convex
functions of their parameters and input, we can write the two-layer
sub-network as $h=f(g_{1}(\boldsymbol{\theta}_{1}),\boldsymbol{\theta}_{2})$.
By repeatedly applying lemma \ref{lemma:composition_multi_convex},
the whole network is multi-convex on a finite number of sets covering
the input and parameter space.

Now we extend the theorem to the whole dataset, where each data point
defines a continuous piecewise multi-convex function $f_{k}$. By
lemma \ref{lemma:sum_piecewise_multi_convex}, the arithmetic mean
$(1/N)\sum_{k=1}^{N}f_{k}$ is continuous piecewise multi-convex.

\end{proof}

In the coming sections, we shall see that multi-convexity allows us
to give certain guarantees about the convergence of various optimization
algorithms. But first, we shall prove some basic results independent
of the optimization procedure. These results were summarized by Gorksi
et al\@.~for the case of biconvex differentiable functions \cite{Gorski:2007:BiConvex}.
Here we extend them to piecewise functions and arbitrary index sets.
First we define a special type of minimum relevant for multi-convex
functions.

\begin{definition}

Let $f:S\rightarrow\mathbb{R}$ and let $\mathcal{I}$ be a collection
of sets covering $\{1,2,...,n\}$. We say that $\boldsymbol{x}_{0}$
is a \textbf{partial minimum} of $f$ with respect to $\mathcal{I}$
if $f(\boldsymbol{x}_{0})\le f(\boldsymbol{x})$ for all $\boldsymbol{x}\in\cup_{I\in\mathcal{I}}S_{I}(\boldsymbol{x}_{0})$.

\end{definition}

In other words, $\boldsymbol{x}_{0}$ is a partial minimum of $f$
with respect to $\mathcal{I}$ if it minimizes $f$ on every cross-section
of $S$ intersecting $\boldsymbol{x}_{0}$, as shown in figure \ref{fig:cross_section}.
By convexity, these points are intimately related to the stationary
points of $f$.
\begin{figure}
\begin{centering}
\includegraphics[scale=0.7]{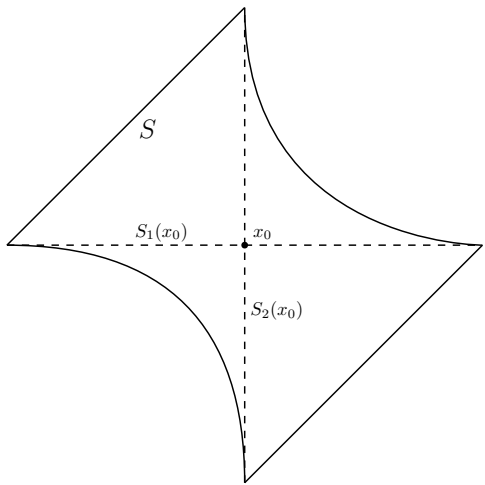}
\par\end{centering}
\caption{Cross-sections of a biconvex set.}
\label{fig:cross_section}
\end{figure}

\begin{theorem}\label{theorem:multi_convex_partial_minimum}

Let $\mathcal{I}=\{I_{1},I_{2},...,I_{m}\}$ be a collection of sets
covering $\{1,2,...,n\}$, let $f:\mathbb{R}^{n}\rightarrow\mathbb{R}$
be continuous piecewise multi-convex with respect to $\mathcal{I}$,
and let $\nabla f(\boldsymbol{x}_{0})=\boldsymbol{0}$. Then $\boldsymbol{x}_{0}$
is a partial minimum of $f$ on every piece containing $\boldsymbol{x}_{0}$.

\end{theorem}

\begin{proof}

Let $S$ be a piece of $f$ containing \textbf{$\boldsymbol{x}_{0}$},
let $I\in\mathcal{I}$ , and let $S_{I}(\boldsymbol{x}_{0})$ denote
the relevant cross-section of $S$. We know $f$ is convex on $S_{I}(\boldsymbol{x}_{0})$,
and since $\nabla f(\boldsymbol{x}_{0})=\boldsymbol{0}$, we have
that $\boldsymbol{x}_{0}$ minimizes $f$ on this convex set. Since
this holds for all $I\in\mathcal{I}$, $\boldsymbol{x}_{0}$ is a
partial minimum of $f$ on $S$.

\end{proof}

It is clear that multi-convexity provides a wealth of results concerning
partial minima, while piecewise multi-convexity restricts those results
to a subset of the domain. Less obvious is that partial minima of
smooth multi-convex functions need not be local minima. An example
was pointed out by a reviewer of this work, that the biconvex function
$f(x,y)=xy$ has a partial minimum at the origin which is not a local
minimum. However, the converse is easily verified, even in the absence
of differentiability.

\begin{theorem}

Let $\mathcal{I}$ be a collection of sets covering $\{1,2,...,n\}$,
let $f:\mathbb{R}^{n}\rightarrow\mathbb{R}$ be continuous piecewise
multi-convex with respect to $\mathcal{I}$, and let $\boldsymbol{x}_{0}$
be a local minimum on some piece $S$ of $f$. Then $\boldsymbol{x}_{0}$
is a partial minimum on $S$.

\end{theorem}

\begin{proof}

The proof is essentially the same as that of theorem \ref{theorem:multi_convex_partial_minimum}.

\end{proof}

We have seen that for multi-convex functions there is a close relationship
between stationary points, local minima and partial minima. For these
functions, infinitesimal results concerning derivatives and local
minima can be extended to larger sets. However, we make no guarantees
about global minima. The good news is that, unlike global minima,
we shall see that we can easily solve for partial minima.

\section{Gradient descent\label{sec:Optimization-and-convergence}}

In the realm of non-convex optimization, also called global optimization,
methods can be divided into two groups: those which can certifiably
find a global minimum, and those which cannot. In the former group
we sacrifice speed, in the latter correctness. This work focuses on
algorithms of the latter kind, called local or sub-optimal methods,
as only this type is used in practice for deep neural networks. In
particular, the most common methods are variants of gradient descent,
where the gradient of the network with respect its parameters is computed
by a procedure called backpropagation. Since its explanation is often
obscured by jargon, we shall provide a simple summary here.

Backpropagation is nothing but the chain rule applied to the layers
of a network. Splitting the network into two functions $f=u\circ v$,
where $u:\mathbb{R}^{n}\rightarrow\mathbb{R}$, and $v:\mathbb{R}^{m}\rightarrow\mathbb{R}^{n}$,
we have
\[
\nabla f=\nabla u\mathcal{D}v
\]
where $\mathcal{D}$ denotes the Jacobian operator. Note that here
the parameters of $u$ are considered fixed, whereas the parameters
of $v$ are variable and the input data is fixed. Thus $\nabla f$
is the gradient of $f$ with respect to the parameters of $v$, if
it exists. The special observation is that we can proceed from the
top layer of the neural network $g_{N}$ to the bottom $g_{1}$, with
$u=g_{N}\circ g_{N-1}\circ...\circ g_{m+1}$, and $v=g_{m}$, each
time computing the gradient of $f$ with respect to the parameters
of $g_{m}$. In this way, we need only store the vector $\nabla u$
and the matrix $\mathcal{D}v$ can be forgotten at each step. This
is known as the ``backward pass,'' which allows for efficient computation
of the gradient of a neural network with respect to its parameters.
A similar algorithm computes the value of $g_{m-1}\circ g_{m-2}\circ...\circ g_{1}$
as a function of the input data, which is often needed to evaluate
$\mathcal{D}v$. First we compute and store $g_{1}$ as a function
of the input data, then $g_{2}\circ g_{1}$, and so on until we have
$f$. This is known as the ``forward pass.'' After one forward and
one backward pass, we have computed $\nabla f$ with respect to all
the network parameters.

Having computed $\nabla f$, we can update the parameters by gradient
descent, defined as follows. 

\begin{definition}

Let $S\subset\mathbb{R}^{n}$, and $f:S\rightarrow\mathbb{R}$ be
partial differentiable, with $\boldsymbol{x}_{0}\in S$. Then \textbf{gradient
descent} on $f$ is the sequence $\{\boldsymbol{x}_{k}\}_{k=0}^{\infty}$
defined by
\[
\boldsymbol{x}_{k+1}=\boldsymbol{x}_{k}-\alpha_{k}\nabla f(\boldsymbol{x}_{k})
\]
where $\alpha_{k}>0$ is called the \textbf{step size} or ``learning
rate.'' In this work we shall make the additional assumption that
$\sum_{k=0}^{\infty}a_{k}=\infty$.

\end{definition}

Variants of this basic procedure are preferred in practice because
their computational cost scales well with the number of network parameters.
There are many different ways to choose the step size, but our assumption
that $\sum_{k=0}^{\infty}a_{k}=\infty$ covers what is usually done
with deep neural networks. Note that we have not defined what happens
if $\boldsymbol{x}_{k}\notin S$. Since we are ultimately interested
in neural networks on $\mathbb{R}^{n}$, we can ignore this case and
say that the sequence diverges. Gradient descent is not guaranteed
to converge to a global minimum for all differentiable functions.
However, it is natural to ask to which points it can converge. This
brings us to a basic but important result.

\begin{theorem}\label{theorem:gradient_descent}

Let $f:\mathbb{R}^{n}\rightarrow\mathbb{R}$, and let $\{\boldsymbol{x}_{k}\}_{k=0}^{\infty}$
result from gradient descent on $f$ with $\lim_{k\rightarrow\infty}\boldsymbol{x}_{k}=\boldsymbol{x}^{*}$,
and $f$ continuously differentiable at $\boldsymbol{x}^{*}$. Then
$\nabla f(\boldsymbol{x}^{*})=\boldsymbol{0}$.

\end{theorem}

\begin{proof}

First, we have
\[
\boldsymbol{x}^{*}=\boldsymbol{x}_{0}-\sum_{k=0}^{\infty}\alpha_{k}\nabla f(\boldsymbol{x}_{k}).
\]
Assume for the sake of contradiction that for the $j^{th}$ partial
derivative we have $|\partial f(\boldsymbol{x}^{*})/\partial(\boldsymbol{x})_{j}|>0$.
Now, pick some $\varepsilon$ such that $0<\varepsilon<|\partial f(\boldsymbol{x}^{*})/\partial(\boldsymbol{x})_{j}|$,
and by continuous differentiability, there is some $\delta>0$ such
that for all $\boldsymbol{x}$, $\|\boldsymbol{x}^{*}-\boldsymbol{x}\|_{2}<\delta$
implies $\|\nabla f(\boldsymbol{x}^{*})-\nabla f(\boldsymbol{x})\|_{2}<\varepsilon$.
Now, there must be some $K$ such that for all $k\ge K$ we have $\|\boldsymbol{x}^{*}-\boldsymbol{x}_{k}\|_{2}<\delta$,
so that $\partial f(\boldsymbol{x}_{k})/\partial(\boldsymbol{x})_{j}$
does not change sign. Then we can write 
\begin{align*}
\left|\sum_{k=K}^{\infty}\alpha_{k}\frac{\partial f(\boldsymbol{x}_{k})}{\partial\left(\boldsymbol{x}\right)_{j}}\right| & =\sum_{k=K}^{\infty}\alpha_{k}\left|\frac{\partial f(\boldsymbol{x}_{k})}{\partial\left(\boldsymbol{x}\right)_{j}}\right|\\
 & \ge\sum_{k=K}^{\infty}\alpha_{k}\left(\left|\frac{\partial f(\boldsymbol{x}^{*})}{\partial\left(\boldsymbol{x}\right)_{j}}\right|-\varepsilon\right)\\
 & =\infty.
\end{align*}
But this contradicts the fact that $\boldsymbol{x}_{k}$ converges.
Thus $\nabla f(\boldsymbol{x}^{*})=\boldsymbol{0}$.

\end{proof}

In the convex optimization literature, this simple result is sometimes
stated in connection with Zangwill's much more general convergence
theorem \cite{Zangwill:1969:Optimization,Iusem:2003:subgradientConvergence}.
Note, however, that unlike Zangwill we state necessary, rather than
sufficient conditions for convergence. While many similar results
are known, it is difficult to strictly weaken the conditions of theorem
\ref{theorem:gradient_descent}. For example, if we relax the condition
that $\alpha_{k}$ is not summable, and take $f(x)=x$, then $x_{k}$
will always converge to a non-stationary point. Similarly, if we relax
the constraint that $f$ is continuously differentiable, taking $f(x)=|x|$
and $a_{k}$ decreasing monotonically to zero, we will always converge
to the origin, which is not differentiable. Furthermore, if we have
$f(x)=|x|$ with $\alpha_{k}$ constant, then $x_{k}$ will not converge
for almost all $x_{0}$. It is possible to prove much stronger necessary
and sufficient conditions for gradient descent, but these results
require additional assumptions about the step size policy as well
as the function to be minimized, and possibly even the initialization
$\boldsymbol{x}_{0}$ \cite{Nesterov:2004:ConvexBook}.

It is worth discussing $f(x)=|x|$ in greater detail, since this is
a piecewise affine function and thus of interest in our investigation
of neural networks. While we have said its only convergence point
is not differentiable, it remains subdifferentiable, and convergence
results are known for subgradient descent \cite{Iusem:2003:subgradientConvergence}.
In this work we shall not make use of subgradients, instead considering
descent on a piecewise continuously differentiable function, where
the pieces are $x\le0$ and $x\ge0$. Although theorem \ref{theorem:gradient_descent}
does not apply to this function, the relevant results hold anyways.
That is, $x=0$ is minimal on some piece of $f$, a result which extends
to any continuous piecewise convex function, as any saddle point is
guaranteed to minimize some piece.

Here we should note one way in which this analysis fails in practice.
So far we have assumed the gradient $\nabla f$ is precisely known.
In practice, it is often prohibitively expensive to compute the average
gradient over large datasets. Instead we take random subsamples, in
a procedure known as \textit{stochastic }gradient descent. We will
not analyze its properties here, as current results on the topic impose
additional restrictions on the objective function and step size, or
require different definitions of convergence \cite{Bertsekas:2010:IncrementalGradient,Bach:2011:sgd,Ge:2015:sgdSaddle}.
Restricting ourselves to the true gradient $\nabla f$ allows us to
provide simple proofs applying to an extensive class of neural networks.

We are now ready to generalize these results to neural networks. There
is a slight ambiguity in that the boundary points between pieces need
not be differentiable, nor even sub-differentiable. Since we are interested
only in necessary conditions, we will say that gradient descent diverges
when $\nabla f(\boldsymbol{x}_{k})$ does not exist. However, our
next theorem can at least handle non-differentiable limit points.

\begin{theorem}\label{theorem:multi_convex_gradient_descent}

Let $\mathcal{I}=\{I_{1},I_{2},...,I_{m}\}$ be a collection of sets
covering $\{1,2,...,n\}$, let $f:\mathbb{R}^{n}\rightarrow\mathbb{R}$
be continuous piecewise multi-convex with respect to $\mathcal{I}$,
and piecewise continuously differentiable. Then, let $\{\boldsymbol{x}_{k}\}_{k=0}^{\infty}$
result from gradient descent on $f$ , with $\lim_{k\rightarrow\infty}\boldsymbol{x}_{k}=\boldsymbol{x}^{*}$,
such that either
\begin{enumerate}
\item $f$ is continuously differentiable at $\boldsymbol{x}^{*}$, or
\item there is some piece $S$ of $f$ and some $K>0$ such that $\boldsymbol{x}_{k}\in\interior S$
for all $k\ge K$.
\end{enumerate}
Then $\boldsymbol{x}^{*}$ is a partial minimum of $f$ on every piece
containing $\boldsymbol{x}^{*}$.

\end{theorem}

\begin{proof}

If the first condition holds, the result follows directly from theorems
\ref{theorem:gradient_descent} and \ref{theorem:multi_convex_partial_minimum}.
If the second condition holds, then $\{\boldsymbol{x}_{k}\}_{k=K}^{\infty}$
is a convergent gradient descent sequence on $g$, the active function
of $f$ on $S$. Since $g$ is continuously differentiable on $\mathbb{R}^{n}$,
the first condition holds for $g$. Since $f|_{S}=g|_{S}$, $\boldsymbol{x}^{*}$
is a partial minimum of $f|_{S}$ as well.

\end{proof}

The first condition of theorem \ref{theorem:multi_convex_gradient_descent}
holds for every point in the interior of a piece, and some boundary
points. The second condition extends these results to non-differentiable
boundary points so long as gradient descent is eventually confined
to a single piece of the function\@. For example, consider the continuous
piecewise convex function $f(x)=\min(x,x^{4})$ as shown in figure
\ref{fig:convergence}. When we converge to $x=0$ from the piece
$[0,1]$, it is as if we were converging on the smooth function $g(x)=x^{4}$.
This example also illustrates an important caveat regarding boundary
points: although $x=0$ is an extremum of $f$ on $[0,1]$, it is
not an extremum on $\mathbb{R}$.

\begin{figure}
\begin{centering}
\includegraphics[scale=0.4]{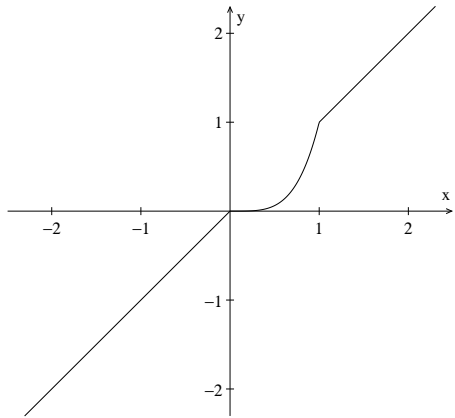}
\par\end{centering}
\centering{}\caption{Example of a piecewise convex function. The point $x=0$ minimizes
the function on the piece $[0,1]$.}
\label{fig:convergence}
\end{figure}

\section{Iterated convex optimization}

Although the previous section contained some powerful results, theorem
\ref{theorem:multi_convex_gradient_descent} suffers from two main
weaknesses, that it is a necessary condition and that it requires
extra care at non-differentiable points. It is difficult to overcome
these limitations with gradient descent. Instead, we shall define
a different optimization technique, from which necessary and sufficient
convergence results follow, regardless of differentiability.

Iterated convex optimization splits a non-convex optimization problem
into a number of convex sub-problems, solving the sub-problems in
each iteration. For a neural network, we have shown that the problem
of optimizing the parameters of a single layer, all others held constant,
is piecewise convex. Thus, restricting ourselves to a given piece
yields a convex optimization problem. In this section, we show that
these convex sub-problems can be solved repeatedly, converging to
a piecewise partial optimum.

\begin{definition}

Let $\mathcal{I}=\{I_{1},I_{2},...,I_{m}\}$ be a collection of sets
covering $\{1,2,...,n\}$, and let $S\subseteq\mathbb{R}^{n}$ and
$f:S\rightarrow\mathbb{R}$ be multi-convex with respect to $\mathcal{I}$.
Then \textbf{iterated convex optimization }is any sequence where $\boldsymbol{x}_{k}$
is a solution to the optimization problem
\begin{align}
\mbox{minimize } & f(\boldsymbol{y})\label{eq:iterated_convex}\\
\mbox{subject to } & \boldsymbol{y}\in\cup_{I\in\mathcal{I}}S_{I}(\boldsymbol{x}_{k-1})\nonumber 
\end{align}
with \textbf{$\boldsymbol{x}_{0}\in S$.}

\end{definition}

We call this iterated convex optimization because problem \ref{eq:iterated_convex}
can be divided into convex sub-problems
\begin{align}
\mbox{minimize } & f(\boldsymbol{y})\label{eq:convex_subproblem}\\
\mbox{subject to } & \boldsymbol{y}\in S_{I}(\boldsymbol{x}_{k-1}).\nonumber 
\end{align}
for each $I\in\mathcal{I}$. In this work, we assume the convex sub-problems
are solvable, without delving into specific solution techniques. Methods
for alternating between solvable sub-problems have been studied by
many authors, for many different types of sub-problems \cite{Wendell:1976:Bilinear}.
In the context of machine learning, the same results have been developed
for the special case of linear autoencoders \cite{Baldi:2012:ComplexValuedAutoencoders}.
Still, extra care must be taken in extending these results to arbitrary
index sets. The key is that $\boldsymbol{x}_{k}$ is not updated until
all sub-problems have been solved, so that each iteration consists
of solving $m$ convex sub-problems. This is equivalent to the usual
alternating convex optimization for biconvex functions, where $\mathcal{I}$
consists of two sets, but not for general multi-convex functions. 

Some basic convergence results follow immediately from the solvability
of problem \ref{eq:iterated_convex}. First, note that $\boldsymbol{x}_{k-1}$
is a feasible point, so we have $f(\boldsymbol{x}_{k})\le f(\boldsymbol{x}_{k-1})$.
This implies that $\lim_{k\rightarrow\infty}f(\boldsymbol{x}_{k})$
exists, so long as $f$ is bounded below. However, this does not imply
the existence of $\lim_{k\rightarrow\infty}\boldsymbol{x}_{k}$. See
Gorski et al.~for an example of a biconvex function on which $\boldsymbol{x}_{k}$
diverges \cite{Gorski:2007:BiConvex}. To prove stronger convergence
results, we introduce regularization to the objective. 

\begin{theorem}\label{theorem:regularization}

Let $\mathcal{I}$ be a collection of sets covering $\{1,2,...,n\}$,
and let $S\subseteq\mathbb{R}^{n}$ and $f:S\rightarrow\mathbb{R}$
be multi-convex with respect to $\mathcal{I}$. Next, let $\inf f>-\infty$,
and let $g(\boldsymbol{x})=f(\boldsymbol{x})+\lambda\|\boldsymbol{x}\|$,
where $\lambda>0$ and $\|\boldsymbol{x}\|$ is a convex norm. Finally,
let $\{\boldsymbol{x}_{k}\}_{k=0}^{\infty}$ result from iterated
convex optimization of $g$. Then $\boldsymbol{x}_{k}$ has at least
one convergent subsequence, in the topology induced by the metric
$d(\boldsymbol{x},\boldsymbol{y})=\|\boldsymbol{x}-\boldsymbol{y}\|$.

\end{theorem}

\begin{proof}

From lemma \ref{lemma:sum_piecewise_convex}, $g$ is multi-convex,
so we are allowed iterated convex optimization. Now, if $\inf f+\lambda\|\boldsymbol{x}\|>g(\boldsymbol{x}_{0})$
we have that $g(\boldsymbol{x})>g(\boldsymbol{x}_{0})$. Thus $g(\boldsymbol{x})>g(\boldsymbol{x}_{0})$
whenever $\|\boldsymbol{x}\|>\left(g(\boldsymbol{x}_{0})-\inf f\right)/\lambda$.
Since $g(\boldsymbol{x}_{k})$ is a non-increasing sequence, we have
that $\|\boldsymbol{x}_{k}\|\le\left(g(\boldsymbol{x}_{0})-\inf f\right)/\lambda$.
Equivalently, $\boldsymbol{x}_{k}$ lies in the set $A=\{\boldsymbol{x}:\|\boldsymbol{x}\|\le\left(g(\boldsymbol{x}_{0})-\inf f\right)/\lambda\}$.
Since $\|\boldsymbol{x}\|$ is continuous, $A$ is closed and bounded,
and thus it is compact. Then, by the Bolzano-Weierstrauss theorem,
$\boldsymbol{x}_{k}$ has at least one convergent subsequence \cite{Johnsonbaugh:1970:RealAnalysis}.

\end{proof}

In theorem \ref{theorem:regularization}, the function $g$ is called
the \textbf{regularized} version of $f$. In practice, regularization
often makes a non-convex optimization problem easier to solve, and
can reduce over-fitting. The theorem shows that iterated convex optimization
on a regularized function always has at least one convergent subsequence.
Next, we shall establish some rather strong properties of the limits
of these subsequences.

\begin{theorem}\label{theorem:sequential_convex_optimization_convergence}

Let $\mathcal{I}$ be a collection of sets covering $\{1,2,...,n\}$,
and let $S\subseteq\mathbb{R}^{n}$ and $f:S\rightarrow\mathbb{R}$
be multi-convex with respect to $\mathcal{I}$. Next, let $\{\boldsymbol{x}_{k}\}_{k=0}^{\infty}$
result from iterated convex optimization of $f$. Then the limit of
every convergent subsequence is a partial minimum on $\interior S$
with respect to $\mathcal{I}$, in the topology induced by the metric
$d(\boldsymbol{x},\boldsymbol{y})=\|\boldsymbol{x}-\boldsymbol{y}\|$
for some norm $\|\boldsymbol{x}\|$. Furthermore, if $\{\boldsymbol{x}_{m_{k}}\}_{k=1}^{\infty}$
and $\{\boldsymbol{x}_{n_{k}}\}_{k=1}^{\infty}$ are convergent subsequences,
then $\lim_{k\rightarrow\infty}f(\boldsymbol{x}_{m_{k}})=\lim_{k\rightarrow\infty}f(\boldsymbol{x}_{n_{k}})$.

\end{theorem}

\begin{proof}

Let $\boldsymbol{x}_{n_{k}}$ denote a subsequence of $\boldsymbol{x}_{k}$
with $\boldsymbol{x}^{*}=\lim_{n\rightarrow\infty}\boldsymbol{x}_{n_{k}}$.
Now, assume for the sake of contradiction that $\boldsymbol{x}^{*}$
is not a partial minimum on $\mbox{int}S$ with respect to $\mathcal{I}$.
Then there is some $I\in\mathcal{I}$ and some $\boldsymbol{x}^{\prime}\in S_{I}(\boldsymbol{x}^{*})$
with $\boldsymbol{x}^{\prime}\in\interior S$ such that $f(\boldsymbol{x}^{\prime})<f(\boldsymbol{x}^{*})$.
Now, $f$ is continuous at $\boldsymbol{x}^{\prime}$, so there must
be some $\delta>0$ such that for all $\boldsymbol{x}\in S$, $\|\boldsymbol{x}-\boldsymbol{x}^{\prime}\|<\delta$
implies $|f(\boldsymbol{x})-f(\boldsymbol{x}^{\prime})|<f(\boldsymbol{x}^{*})-f(\boldsymbol{x}^{\prime})$.
Furthermore, since $\boldsymbol{x}^{\prime}$ is an interior point,
there must be some open ball $B\subset S$ of radius $r$ centered
at $\boldsymbol{x}^{\prime}$, as shown in figure \ref{fig:proof_iterated_convex_convergence}.
Now, there must be some $K$ such that $\|\boldsymbol{x}_{n_{K}}-\boldsymbol{x}^{*}\|<\min(\delta,r)$.
Then, let $\tilde{\boldsymbol{x}}=\boldsymbol{x}_{n_{K}}+\boldsymbol{x}^{\prime}-\boldsymbol{x}^{*}$,
and since $\|\tilde{\boldsymbol{x}}-\boldsymbol{x}^{\prime}\|<r$,
we know that $\tilde{\boldsymbol{x}}\in B$, and thus $\tilde{\boldsymbol{x}}\in S_{I}(\boldsymbol{x}_{n_{K}})$.
Finally, $\|\tilde{\boldsymbol{x}}-\boldsymbol{x}^{\prime}\|<\delta$,
so we have $f(\tilde{\boldsymbol{x}})<f(\boldsymbol{x}^{*})\le f(\boldsymbol{x}_{n_{K}+1})$,
which contradicts the fact that $\boldsymbol{x}_{n_{K}+1}$ minimizes
$g$ over a set containing $\tilde{\boldsymbol{x}}$. Thus $\boldsymbol{x}^{*}$
is a partial minimum on $\interior S$ with respect to $\mathcal{I}$.

Finally, let $\{\boldsymbol{x}_{m_{k}}\}_{k=1}^{\infty}$ and $\{\boldsymbol{x}_{n_{k}}\}_{k=1}^{\infty}$
be two convergent subsequences of $\boldsymbol{x}_{k}$, with $\lim_{k\rightarrow\infty}\{\boldsymbol{x}_{m_{k}}\}=\boldsymbol{x}_{m}^{*}$
and $\lim_{k\rightarrow\infty}\{\boldsymbol{x}_{n_{k}}\}=\boldsymbol{x}_{n}^{*}$,
and assume for the sake of contradiction that $f(\boldsymbol{x}_{m}^{*})>f(\boldsymbol{x}_{n}^{*})$.
Then by continuity, there is some $K$ such that $f(\boldsymbol{x}_{n_{K}})<f(\boldsymbol{x}_{m}^{*})$.
But this contradicts the fact that $f(\boldsymbol{x}_{k})$ is non-increasing.
Thus $f(\boldsymbol{x}_{m}^{*})=f(\boldsymbol{x}_{n}^{*})$.

\end{proof}

\begin{figure}
\begin{centering}
\includegraphics[scale=0.5]{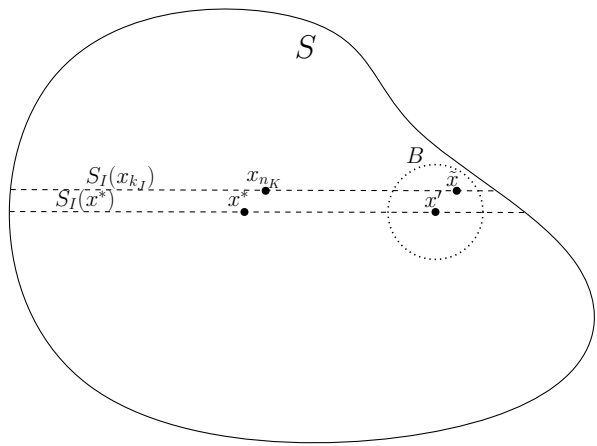}
\par\end{centering}
\caption{Illustration of the proof of theorem \ref{theorem:sequential_convex_optimization_convergence}.
Note the cross-sections of the biconvex set $S$.}
\label{fig:proof_iterated_convex_convergence}
\end{figure}

The previous theorem is an extension of results reviewed in Gorski
et al.~to arbitrary index sets \cite{Gorski:2007:BiConvex}. While
Gorski et al.~explicitly constrain the domain to a compact biconvex
set, we show that regularization guarantees $\boldsymbol{x}_{k}$
cannot escape a certain compact set, establishing the necessary condition
for convergence. Furthermore, our results hold for general multi-convex
sets, while the earlier result is restricted to Cartesian products
of compact sets.

These results for iterated convex optimization are considerably stronger
than what we have shown for gradient descent. While any bounded sequence
in $\mathbb{R}^{n}$ has a convergent subsequence, and we can guarantee
boundedness for some variants of gradient descent, we cannot normally
say much about the limits of subsequences. For iterated convex optimization,
we have shown that the limit of any subsequence is a partial minimum,
and all limits of subsequences are equal in objective value. For all
practical purposes, this is just as good as saying that the original
sequence converges to partial minimum.

\section{Global optimization\label{sec:Local-minima}}

Although we have provided necessary and sufficient conditions for
convergence of various optimization algorithms on neural networks,
the points of convergence need only minimize cross-sections of pieces
of the domain. Of course we would prefer results relating the points
of convergence to global minima of the training objective. In this
section we illustrate the difficulty of establishing such results,
even for the simplest of neural networks.

In recent years much work has been devoted to providing theoretical
explanations for the empirical success of deep neural networks, a
full accounting of which is beyond the scope of this article. In order
to simplify the problem, many authors have studied \textit{linear}
neural networks, in which the layers have the form $g(\boldsymbol{x})=A\boldsymbol{x}$,
where $A$ is the parameter matrix. With multiple layers this is clearly
a linear function of the output, but not of the parameters. As a special
case of piecewise affine functions, our previous results suffice to
show that these networks are multi-convex as functions of their parameters.
This was proven for the special case of linear autoencoders by Baldi
and Lu \cite{Baldi:2012:ComplexValuedAutoencoders}.

Many authors have claimed that linear neural networks contain no ``bad''
local minima, i.e.~every local minimum is a global minimum \cite{Kawaguchi:2016:WithoutPoorLocalMinima,Soudry:2016:NoBadLocalMinima}.
This is especially evident in the study of linear autoencoders, which
were shown to admit many points of inflection, but only a single strict
minimum \cite{Baldi:2012:ComplexValuedAutoencoders}. While powerful,
this claim does not apply to the networks seen in practice. To see
this, consider the dataset $D=\{(0,1/2),(-1,\alpha),(1,2\alpha)\}$
consisting of three $(x,y)$ pairs, parameterized by $\alpha>1$.
Note that the dataset has zero mean and unit variance in the $x$
variable, which is common practice in machine learning. However, we
do not take zero mean in the $y$ variable, as the model we shall
adopt is non-negative. 

Next, consider the simple neural network
\begin{align}
f(a,b) & =\sum_{(x,y)\in D}\left(y-\left[ax+b\right]_{+}\right)^{2}\label{eq:single_layer_objective}\\
 & =\left(\frac{1}{2}-\left[b\right]_{+}\right)^{2}+\left(\alpha-\left[b-a\right]_{+}\right)^{2}+\left(2\alpha-\left[b+a\right]_{+}\right)^{2}.\nonumber 
\end{align}
This is the squared error of a single ReLU neuron, parameterized by
$(a,b)\in\mathbb{R}^{2}$. We have chosen this simplest of all networks
because we can solve for the local minima in closed form, and show
they are indeed very bad. First, note that $f$ is a continuous piecewise
convex function of six pieces, realized by dividing the plane along
the line $ax+b=0$ for each $x\in D$, as shown in figure \ref{fig:bad_local_minimum}.
Now, for all but one of the pieces, the ReLU is ``dead'' for at
least one of the three data points, i.e.~$ax+b<0$. On these pieces,
at least one of the three terms of equation \ref{eq:single_layer_objective}
is constant. The remaining terms are minimized when $y=ax+b$, represented
by the three dashed lines in figure \ref{fig:bad_local_minimum}.
There are exactly three points where two of these lines intersect,
and we can easily show that two of them are strict local minima. Specifically,
the point $(a_{1},b_{1})=(1/2-\alpha,1/2)$ minimizes the first two
terms of equation \ref{eq:single_layer_objective}, while $(a_{2},b_{2})=(2\alpha-1/2,1/2)$
minimizes the first and last term. In each case, the remaining term
is constant over the piece containing the point of intersection. Thus
these points are strict global minima on their respective pieces,
and strict local minima on $\mathbb{R}^{2}$. Furthermore, we can
compute $f(a_{1},b_{1})=4\alpha^{2}$ and $f(a_{2},b_{2})=\alpha^{2}$.
This gives
\begin{align*}
\lim_{\alpha\rightarrow\infty}a_{1} & =-\infty,\\
\lim_{\alpha\rightarrow\infty}a_{2} & =+\infty,
\end{align*}
and
\[
\lim_{\alpha\rightarrow\infty}\left(f(a_{1},b_{1})-f(a_{2},b_{2})\right)=\infty.
\]
Now, it might be objected that we are not permitted to take $\alpha\rightarrow\infty$
if we require that the $y$ variable has unit variance. However, these
same limits can be achieved with variance tending to unity by adding
$\left\lfloor \alpha\right\rfloor $ instances of the point $(1,2\alpha)$
to our dataset. Thus even under fairly stringent requirements we can
construct a dataset yielding arbitrarily bad local minima, both in
the parameter space and the objective value. This provides some weak
justification for the empirical observation that success in deep learning
depends greatly on the data at hand.

We have shown that the results concerning local minima in linear networks
do not extend to the nonlinear case. Ultimately this should not be
a surprise, as with linear networks the problem can be relaxed to
linear regression on a convex objective. That is, the composition
of all linear layers $g(\boldsymbol{x})=A_{1}A_{2}...A_{n}\boldsymbol{x}$
is equivalent to the function $f(\boldsymbol{x})=A\boldsymbol{x}$
for some matrix $A$, and under our previous assumptions the problem
of finding the optimal $A$ is convex. Furthermore, it is easily shown
that the number of parameters in the relaxed problem is polynomial
in the number of original parameters. Since the relaxed problem fits
the data at least as well as the original, it is not surprising that
the original problem is computationally tractable. 
\begin{figure}
\begin{centering}
\includegraphics[scale=0.5]{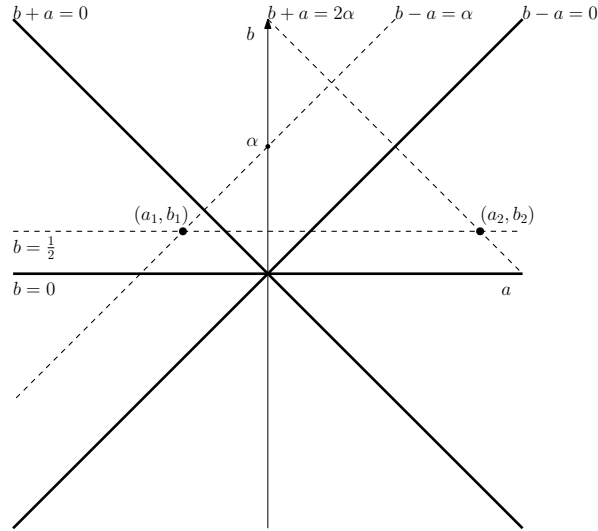}
\par\end{centering}
\caption{Parameter space of the neural network from equation \ref{eq:single_layer_objective},
with pieces divided by the bold black lines. The points $(a_{1},b_{1})$
and $(a_{2},b_{2})$ are local minima, which can be made arbitrarily
far apart by varying the dataset.}
\label{fig:bad_local_minimum}
\end{figure}

This simple example was merely meant to illustrate the difficulty
of establishing results for \textit{every} local minimum of \textit{every}
neural network. Since training a certain kind of network is known
to be NP-Complete, it is difficult to give any guarantees about worst-case
global behavior \cite{Blum:1992:OCT:148433.148441}. We have made
no claims, however, about probabilistic behavior on the average practical
dataset, nor have we ruled out the effects of more specialized networks,
such as very deep ones.

\section{Conclusion}

We showed that a common class of neural networks is piecewise convex
in each layer, with all other parameters fixed. We extended this to
a theory of a piecewise multi-convex functions, showing that the training
objective function can be represented by a finite number of multi-convex
functions, each active on a multi-convex parameter set. From here
we derived various results concerning the extrema and stationary points
of piecewise multi-convex functions. We established convergence conditions
for both gradient descent and iterated convex optimization on this
class of functions, showing they converge to piecewise partial minima.
Similar results are likely to hold for a variety of other optimization
algorithms, especially those guaranteed to converge at stationary
points or local minima. 

We have witnessed the utility of multi-convexity in proving convergence
results for various optimization algorithms. However, this property
may be of practical use as well. Better understanding of the training
objective could lead to the development of faster or more reliable
optimization methods, heuristic or otherwise. These results may provide
some insight into the practical success of sub-optimal algorithms
on neural networks. However, we have also seen that local optimality
results do not extend to global optimality as they do for linear autoencoders.
Clearly there is much left to discover about how, or even if we can
optimize deep, nonlinear neural networks.

\section*{Acknowledgments}

The author would like to thank Mihir Mongia for his helpful comments
in preparing this manuscript.

\section*{Funding}

This research did not receive any specific grant from funding agencies
in the public, commercial, or not-for-profit sectors.

\bibliographystyle{elsarticle-num}
\bibliography{nn2016}

\end{document}